\documentclass[conference]{IEEEtran}
\IEEEoverridecommandlockouts
\usepackage{cite}
\usepackage{amsmath,amssymb,amsfonts}
\usepackage{algorithmic}
\usepackage{graphicx}
\usepackage{textcomp}
\usepackage{xcolor}

\usepackage{graphicx}

\usepackage{bm}
\usepackage{amssymb, amsmath}
\usepackage{wrapfig}
\usepackage[ruled,vlined]{algorithm2e}
\usepackage{comment}
\usepackage{multirow}
\usepackage{url}
\usepackage{pdfpages}
\usepackage{bbm}
\usepackage {booktabs}
\usepackage {colortbl,array,xcolor}

\newtheorem{lemma}{Lemma}
\newtheorem{proof}{Proof}

\def\BibTeX{{\rm B\kern-.05em{\sc i\kern-.025em b}\kern-.08em
    T\kern-.1667em\lower.7ex\hbox{E}\kern-.125emX}}
\begin{document}

\title{Deep Graph Attention Networks
}

\author{\IEEEauthorblockN{Blind Review}
\author{\IEEEauthorblockN{Jun Kato}
\IEEEauthorblockA{\textit{ Kwansei Gakuin University} \\
\textit{Graduate School of Science and Technology}\\
Sanda, Japan \\
gpn06467@kwansei.ac.jp}
\and
\IEEEauthorblockN{Airi Mita, Keita Gobara, Akihiro Inokuchi}
\IEEEauthorblockA{\textit{ Kwansei Gakuin University} \\
\textit{School of Science and Technology}\\
Sanda, Japan \\
inokuchi@kwansei.ac.jp}
}
}

\maketitle

\begin{abstract}
Graphs are useful for representing various real-world objects. However, 
graph neural networks (GNNs) tend to suffer from over-smoothing, where the representations of nodes of different classes become similar as the number of layers increases, leading to performance degradation. 
A method that does not require protracted tuning of the number of layers is needed to effectively construct a graph attention network (GAT), a type of GNN.
Therefore, we introduce a method called ``DeepGAT'' for predicting the class to which nodes belong in a deep GAT.
It avoids over-smoothing in a GAT by ensuring that nodes in different classes are not similar at each layer. 
Using DeepGAT to predict class labels, a 15-layer network is constructed without the need to tune the number of layers. 
DeepGAT prevented over-smoothing and achieved a 15-layer GAT with similar performance to a 2-layer GAT, as indicated by the similar attention coefficients.
DeepGAT enables the training of a large network to acquire similar attention coefficients to a network with few layers. 
It avoids the over-smoothing problem and obviates the need to tune the number of layers, thus saving time and enhancing GNN performance.
\end{abstract}


\section{Introduction}
The graph is a useful data structure to represent various objects in the real world. For example, accounts and links in social networks correspond to nodes and edges in graphs, respectively, and relationships in the networks are represented as graphs. Atoms and chemical bonds in molecules also correspond to nodes and edges in graphs, respectively, and the molecules are thus represented as graphs. Various other objects, such as hyperlink structures and function calls in computer programs, can also be represented by graphs. Because various objects can be stored as graph data, technologies for analyzing such data have been attracting a great deal of attention in recent years~\cite{kipf2016semi,xu2018powerful,velivckovic2017graph}.

The graph neural network (GNN)~\cite{xu2018powerful} is a representative deep learning method for graph data. GNNs convolve features of neighboring nodes of a node $v$ to the representation of $v$. GNNs with $L$ layers repeat this convolution $L$ times, and obtain the representation $\bm{h}_v$ of $v$. Once each node is represented in vector form, various conventional machine learning methods such as classification, clustering, and regression can be applied to the graph data, including for node classification, node clustering, and link prediction.

Representative GNNs include the graph convolutional network (GCN)~\cite{kipf2016semi} and graph attention network (GAT)~\cite{velivckovic2017graph}. 
Figure~\ref{convolution} shows the difference between convolutions of GCN and GAT; the features of nodes in the darker blue background greatly affect the representation of $v$. 
For each node $v$, GCN uniformly convolves the features with $l$ steps from $v$ to the representation of $v$, whereas GAT non-uniformly convolves the features to the representation of $v$ using the attention coefficient. Thus, GCN convolves the features of red nodes in Fig.~\ref{convolution} to the representation of $v$, whereas GAT does not always convolve the features to the representation of $v$, which results in different representations of the structure around $v$ learned by GCNs and GATs.
\begin{figure*}[tb]
	\begin{center}
		\includegraphics[width=150mm]{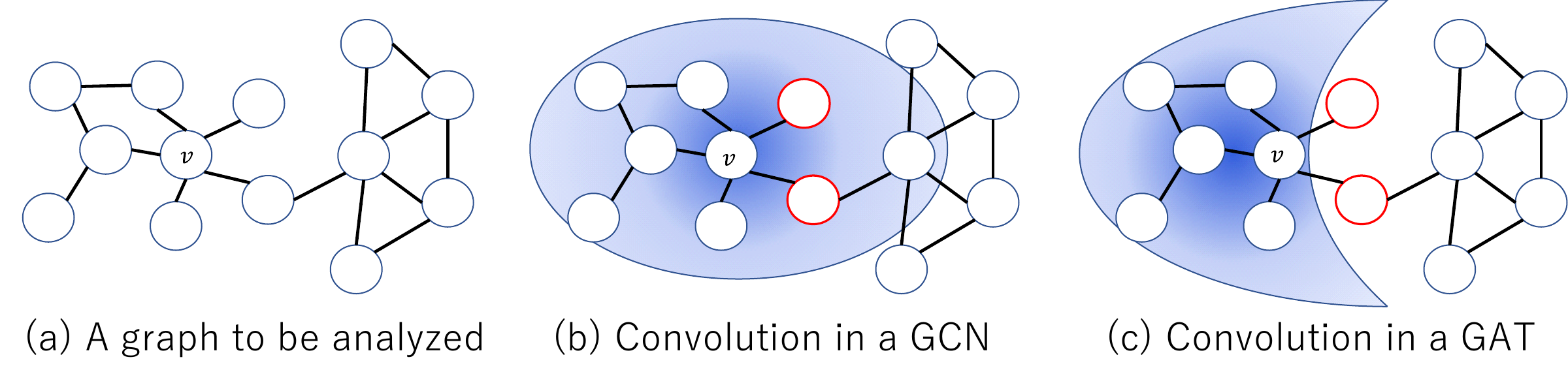}
	\end{center}
        \label{convolution}
	\caption{Convolutions in graph neural networks.}
\end{figure*}

Most GNNs, including GCNs and GATs, have a common drawback called over-smoothing~\cite{Oversmoothing1,Oversmoothing2,Oversmoothing3,Oversmoothing4}, which is a phenomenon in which the representations of nodes belonging to different classes are similar to each other when the number of layers is increased. This results in degradation of GNN performance.
For this reason, low-layer GCNs or GATs with 2-3 layers are often used in actual operations. However, for some graph datasets, even a deep GAT can provide high performance, so we must tune the number of layers $L$ by starting with a low-layer GAT and gradually increasing $L$. This tuning requires considerable time.
Therefore, a methodology that can construct a deep GAT for many datasets without spending time tuning $L$ would be useful.

To achieve this, this paper~\cite{kato-GCA} proposes a method for predicting which class each node belongs to at each layer in a deep GAT, in which the prediction error at each layer is incorporated into the loss function. 
On the basis of the prediction, if a node $v$ has neighbors that belong to the same class as $v$, only features of the neighbors are convolved into the representation of $v$.
Unlike the case of over-smoothing, the likelihood that features of nodes belonging to a different class from $v$ are convolved into $v$ is small, and thus the representations of nodes belonging to different classes are not similar to each other; therefore, over-smoothing can be avoided.

The contributions of this paper are summarized as follows: 
\begin{itemize}
\item 
First, our proposed method, called DeepGAT, achieves a 15-layer GAT with the same level of performance as a 2-layer GAT. Instead of starting with a small $L$ and gradually increasing its value, the DeepGAT with 15 layers can be constructed without tuning $L$.
\item 
Second, we mathematically show why DeepGAT can avoid over-smoothing, on the basis of the reproductive property of the probability distribution (see Lemma~\ref{corollary1}).
We also show mathematically that the overlap in distributions of representations of nodes belonging to different classes is small, and we demonstrate that this is also true for benchmark datasets.
\end{itemize}

\section{Preliminaries}
A graph whose nodes have features is represented as $G=(V,E,X)$, where $V=\{1,2,\ldots, n\}$ is a set of nodes, $E \subseteq V\times V$ is a set of edges, and $X=\{\bm{x}_1, \bm{x}_2,\ldots,\bm{x}_n\} \in \mathbb{R}^{d \times n}$ is a set of feature vectors. We define a neighborhood of vertex $v$ as $N_v=\{u \mid (v,u)\in E\} \cup \{v\}$. Each node belongs to one of two classes $C=\{0,1\}$\footnote{In this paper, we address the binary classification of nodes, although the proposed method can be extended to multi-class node classification, multi-label node classification, and graph classification.}, and we denote nodes belonging to the class $c\in C$ as $V_c \subset V$. In addition, we assume that feature $\bm{x}_v$ of node $v\in V_c$ has the following properties: $\mathcal{N}(\bm{\mu}_c,\Sigma_c)$.

We call a sequence of edges $\langle (v_1,v'_1),(v_2,v'_2),\ldots,(v_l,v'_l)\rangle$ satisfying $v'_i=v_{i+1}$ a path, and $l$ the length of the path. Paths in this paper are not always simple, and we allow the path to satisfy $v_i=v'_i$. We denote a set of paths of length $l$ from node $v$ to node $u$ as $P_{v \sim u}^l$. In addition, we denote a vector representing the numbers of paths of length $l$ from $v$ as $\bm{\pi}_v^l=\sum_{u \in V} |P_{v \sim u}^l|\bm{1}_{n,u}=(|P_{v \sim 1}^l|,|P_{v \sim 2}^l|,\ldots,|P_{v \sim n}^l|)^T$, where $\bm{1}_{n,u}$ is the $n$-dimensional one-hot vector whose $u$-th element is 1. By using these notations, we define $\bm{\pi}_{v,c}^l=\sum_{u \in V_c} |P_{v \sim u}^l|\bm{1}_{n,u}$ and $P_{v}^l = \bigcup _{c \in C} \bigcup _{u\in V_c}P_{v \sim u}^l =\bigcup _{c \in C}P_{v,c}^l$.

Given a graph $G$ and the ground-truth class labels $C'$ for nodes $V' \subset V$ as input, the problem addressed in this paper is the accurate prediction of class labels for nodes $V\setminus V'$ via learning node representations for $V$ from $G$ and $C'$.
\section{Related Works and their Drawbacks}
We will now discuss GAT, which is a type of GNN in which node representations are obtained from a graph. For each node $v$ in a graph, GAT convolves the features of nodes $N_v$ to the feature of $v$. The convolution of GAT consists of {\it Aggregate} and {\it Combine}, as shown below:
\begin{eqnarray}
\bm{a}_v^l &=& {\it Aggregate}^l( \{ \bm{h}_u^{l-1} ~|~ u \in N_v \} ) =\sum_{u \in N_v} \alpha_{vu}\bm{h}_u^{l-1}, \label{agg} \\
\bm{h}_v^l &=& {\it Combine}^l( \bm{h}_v^{l-1}, \bm{a}_v^l ) = \sigma(W^l \bm{a}_v^l), \label{comb} 
\end{eqnarray}
where $W^l$ is the learnable matrix and $\sigma$ is the activation function. In addition, $\alpha_{vu}$ is the attention coefficient computed from $\bm{h}_v^{l-1}$ and $\bm{h}_u^{l-1}$. Starting from $\bm{h}_v^0 =\bm{x}_v $, GAT repeats this convolution $L$ times and then outputs $\bm{h}_v^L$ as the representation for node $v$~\footnote{Although GATs usually use multi-head attention, we discuss GATs with single-head attention until Section~4 of this paper. The software used in Section~5 is implemented with $K$ multi-head attention.}.

GAT is an extension of GCN. $\alpha^l_{vu}$ in GCN does not depend on $\bm{h}_v^l$ and $\bm{h}_u^l$, and $\alpha^l_{vu}$ is $\frac{1}{|N_v|}$. In contrast, GAT aggregates features of neighboring nodes with similar features to node $v$ into a representation of $v$. Thus, it does not convolve uniform features of nodes with $l$ steps from $v$, but instead convolves while selecting features of nodes in particular directions, as shown in Fig.~\ref{convolution}.
The following equations are used for computing attention coefficients between two vectors $\bm{q}$ and $\bm{k}$:
	\begin{eqnarray}
		{\it AD}(\bm{q},\bm{k})&=&\langle \bm{w} , \bm{q} || \bm{k} \rangle ~~~\mbox{(additive attention~\cite{AD-Attention1,velivckovic2017graph,Kato1-Wang})}  \nonumber  \\
		{\it DP}(\bm{q},\bm{k})&=&\langle \bm{q} , \bm{k} \rangle   ~~~\mbox{(dot-product attention~\cite{AttentionVaswani,DP-Attention})},  \nonumber  \\
		{\it SD}(\bm{q},\bm{k})&=&\frac{\langle \bm{q} , \bm{k} \rangle }{dim(\bm{q})}  ~~~\mbox{(scaled dot-product attention)},  \nonumber  
	\end{eqnarray}
where $\bm{w}$ is a learnable vector, $dim$ is a function to obtain the number of dimensions of a vector, $\langle \cdot, \cdot \rangle$ is a dot product, and $||$ is an operation to concatenate two vectors. Using one of these attention functions, $\alpha_{vu}$ is computed by
	\begin{eqnarray}
		\overline{\alpha}_{vu}^l	&=& 	attention(W^l\bm{h}_v^l, W^l\bm{h}_u^l)   \nonumber  \\
		\alpha_{vu}^l			&=& 	\frac{\exp(\overline{\alpha}_{vu})}{\sum_{w \in N_v}\exp(\overline{\alpha}_{vw}) } ~~~\mbox{(Softmax)}.  \nonumber 
	\end{eqnarray}
 
GAT has been actively applied to various methods.
For example, the gated attention network~\cite{Kato2-Zhang} calculates importance values of $K$ multi-heads.
Heterogenous GANs~\cite{Kato1-Wang} can learn heterogeneous graphs with GAT, and 
the signed GAN~\cite{Kato4-Huang} combines directed code networks and GAT.
The relational GAN~\cite{Kato3-Wang} adds new attention networks that can learn the dependencies between emotional polarity and opinion words, called aspects, in sentences. 

One of the major challenges for GNNs, including GATs, is over-smoothing. Over-smoothing is a phenomenon in which node representations $\bm{h}_v^L=\bm{h}_v$ are similar to each other when the number of layers $L$ in the GNN is increased, resulting in a decrease in classification accuracy.
Yajima and Inokuchi proved with the following lemma that the distribution followed by $\bm{h}_v^l$ in a GCN is normal, using the reproductive property of probability distributions.
\begin{lemma}[\cite{yajima}] \label{lemma:gcn-reproductive}
We assume that $\sigma$ in a GCN is the identity function. The representations $\bm{h}^l_v$ in the GCN follow the normal distribution defined below.
\begin{eqnarray}
&&\mathcal{N}\left(\frac{|P_{v,0}^l|}{|P_{v}^l|}W_{1\sim l}\bm{\mu}_0+\frac{|P_{v,1}^l|}{|P_{v}^l|}W_{1\sim l}\bm{\mu}_1, \right. \nonumber \\
&&\left. \frac{\|\bm{\pi}_{v,0}^l\|_2^2}{|P_v^l|^2}W_{1\sim l}^T \Sigma_0 W_{1\sim l}+\frac{\|\bm{\pi}_{v,1}^l\|_2^2}{|P_v^l|^2}W_{1\sim l}^T\Sigma_1 W_{1\sim l} \right),
\nonumber 
\end{eqnarray}
where $W_{1 \sim l}=W^lW^{l-1}\cdots W^1$. $\blacksquare$
\end{lemma}
According to Lemma~\ref{lemma:gcn-reproductive}, the mean $\bm{\mu}_L$ of the distribution followed by $\bm{h}_v=\bm{h}_v^L$ to be used for classification is $\frac{|P_{v,0}^L|\bm{\mu}_0+|P_{v,1}^L|\bm{\mu}_1}{|P_{v,0}^L|+|P_{v,1}^L|}W_{1\sim L} $. 
For large $L$, because the ratio $\frac{|P_{v,0}^L|}{|P_{v,1}^L|}$ approaches the class ratio $\frac{|V_0|}{|V_1|}$, $\bm{\mu}_L$ internally divides $\bm{\mu}_0$ and $\bm{\mu}_1$ in the ratio.
Because the representations of nodes belonging to the two classes follow a distribution with an identical mean $\bm{\mu}_L$ for large $L$, they are similar to each other.
This also holds true for GATs. Therefore, GCNs and GATs often have two or three layers.

The objective of this paper is to extend GAT, i.e., to derive a multiple-layer GAT to achieve high classification accuracy while preventing over-smoothing, by taking advantage of the properties of the representations shown in Lemma~\ref{lemma:gcn-reproductive}.

According to the article~\cite{G2Survey}, there are three main approaches to reduce over-smoothing: Normalization, Residual connection, and Change of GNN Dynamics.
Typical methods of the first approach are DropEdge~\cite{DropEdge} and PairNorm~\cite{PairNorm}.  DropEdge randomly removes a certain number of edges from the input graph at each training epoch. PairNorm transfoms input features in prepossessing as follows. 
\begin{eqnarray}
\hat{\bm{x}}_v&=&\bm{x}_v-\frac{1}{|V|}\sum_{u \in V}\bm{x}_u, \nonumber \\
\bm{x}_v&=&\frac{s\hat{\bm{x}}_v}{\sqrt{\frac{1}{|V|}\sum_{u \in V}||\hat{\bm{x}}_u||_2^2}}, \nonumber 
\end{eqnarray}
where $s>0$ is the hyper-parameter. 
In the second approach,
the residual connection~\cite{ResNet, ResidualGCN} allows the input to bypass one or more layers and be added directly to the output of a later layer. One of implementations is represented as $\bm{x}_v^l=\bm{x}_v^{l-1}+\sigma(W^l\bm{a}_v^l)$. This technique helps mitigate the vanishing gradient problem, enabling the training of deeper networks. GCNII~\cite{GCNII} using residual connection demonstrated that the over-smoothing is reduced in 64-layer GNNs.
The third approach qualitatively changes the dynamics of message passing propagation. 
For example, in Gradient Gating (G$^2$)~\cite{G2}, local gradients due to differences between features are harnessed to further modulate convolution as follows.
\begin{eqnarray}
\hat{\bm{\tau}}_v^l&=&\sigma(W^l\bm{a}_v^l), \nonumber \\
\tau_{vk}^l&=&\tanh\left(\sum_{u \in N_v}|\hat{\tau}_{v}^l-\hat{\tau}_{u}^l|^p\right), \nonumber \\
\bm{x}_v^l&=&(1-\bm{\tau}^l_v)\odot\bm{x}_v^{l-1}+\bm{\tau}^l_v\odot\sigma(W^l\bm{a}_v^l), \nonumber 
\end{eqnarray}
where $p\le 0$ is the hyper-parameter and $\odot$ is the Hadamard product.

\section{DeepGAT}
\subsection{GAT with Oracle Predicting $\overline{\alpha}_{vu}$}
\label{subsec:4A}
In this subsection, we assume that there exists the following oracle $o$ that can output exactly whether two nodes $v$ and $u \in N_v$ belong to the same class.
\[
o(v,u) = \left\{
\begin{array}{ll}
1 & (\mbox{if nodes $v$ and $u$ belong to the same class})\\
0 & (\mbox{otherwise}).
\end{array}
\right.
\]
For node $v$, only features of nodes belonging to the same class as $v$ are convolved into the representation of $v$, when we use $\overline{\alpha}_{vu}^l=o(v,u)$. In this case, we obtain the following lemma.
\begin{lemma}
\label{corollary1}
Let paths consisting of only nodes belonging to the same class as $v$ among $P_{v \sim u}^l$ be $\check{P}_{v \sim u}^l$. We define 
$\check{\bm{\pi}}_{v,c}^l$ as $\sum_{u \in V_c} |\check{P}_{v \sim u}^l| \bm{1}_{n,u} $. When the activation functions are the identify functions, the representation $\bm{h}^l_v $ of $v$ belonging to the class $c$ in GAT with the oracle follows
\begin{eqnarray}
\mathcal{N}\left(W_{1\sim l}\bm{\mu}_c, \frac{\|\check{\bm{\pi}}_{v,c}^l\|_2^2}{\|\check{\bm{\pi}}_{v,c}^l\|_1^2}W_{1\sim l}^T\Sigma_c W_{1\sim l} \right).~\blacksquare
\label{eq:corollary1}
\end{eqnarray}
\end{lemma}
\begin{proof}
The linear sum $a\bm{q}+b\bm{q}'$ of random variables $\bm{q}$ and $\bm{q}'$ that follows the identical distribution $\mathcal{N}(\bm{\mu},\Sigma)$ also follows $\mathcal{N}((a+b)\bm{\mu}, (a^2+b^2)\Sigma)$. Therefore, because Eq.~(\ref{agg}) aggregates $|N_v \cap V_c|$ features that follow an identical distribution in the first layer of GAT, in which the oracle controls attention coefficients, $\bm{a}_v^1$ follows $\mathcal{N}(\bm{\mu}_c, \frac{1}{|N_v \cap V_c|}\Sigma_c)$. Because the linear transformation of random variables that follows the normal distribution generates random variables that follow another normal distribution, $\bm{h}_v^1$ follows $ \mathcal{N}(W^1\bm{\mu}_c, \frac{1}{|N_v \cap V_c|}{W^1}^T\Sigma_c W^1)$. In the $l$-layer GAT with the oracle, the number of features convolved to the representation of $v$ is consistent with the number of paths $|\check{P}_v^l|=|\bigcup_{u \in V} \check{P}_{v \sim u}^l|=\| \check{\bm{\pi}}_{v,c}^l \|_1$. Therefore, the distribution that $\bm{h}^l_v $ follows is expressed as Eq. (\ref{eq:corollary1}). $\Box$
\end{proof}

The mean of the distribution followed by representations of nodes belonging to the class $c$ never depends on the mean $\bm{\mu}_{\overline{c}}$ for the other class $\overline{c}$. When $W^l$ is orthogonal, the variance of Eq.~(\ref{eq:corollary1}) becomes $ \frac{\|\check{\bm{\pi}}_{v,c}^l\|_2^2}{\|\check{\bm{\pi}}_{v,c}^l\|_1^2}\Sigma_c$.
$0<\frac{\|\check{\bm{\pi}}_{v,c}^l\|_2^2}{\|\check{\bm{\pi}}_{v,c}^l\|_1^2} \le 1$ holds, and the variance decreases with an increase of $l$. Because representations $\bm{h}_v^l$ of nodes $v$ belonging to the class $c$ follow a different distribution from the distribution that $\bm{h}_u^l$ of nodes $u$ belonging to $\overline{c}$ follow, and as the variance in the distributions decreases with an increase of $l$, increasing $L$ in turn increases the likelihood that the GAT with the oracle can classify correctly.

Fig.~\ref{variance0} shows the distribution of the input feature vectors $\bm{x}_v=\bm{h}_v^0$. Two bell-shaped distributions are class 0 and class~1, and the gray areas are the overlap of those distributions. The overlap is related to the Bayes error rate.
The Bayes error rate is the lowest possible error rate for any classifier, representing the limit of classification accuracy given the inherent noise in the data. It is the probability that a randomly chosen data point is misclassified by an optimal classifier that knows the true underlying distributions.
Under the presence of oracles, the variances of the distributions of $\bm{h}_v^l$ are reduced by convolving $l$ times, and the overlap of the distributions is correspondingly reduced. Therefore, at each layer, if we can correctly predict the class to which each node belongs, we can expect to reduce the over-smoothing.

\begin{figure}[b]
	\begin{center}
		\includegraphics[width=85mm]{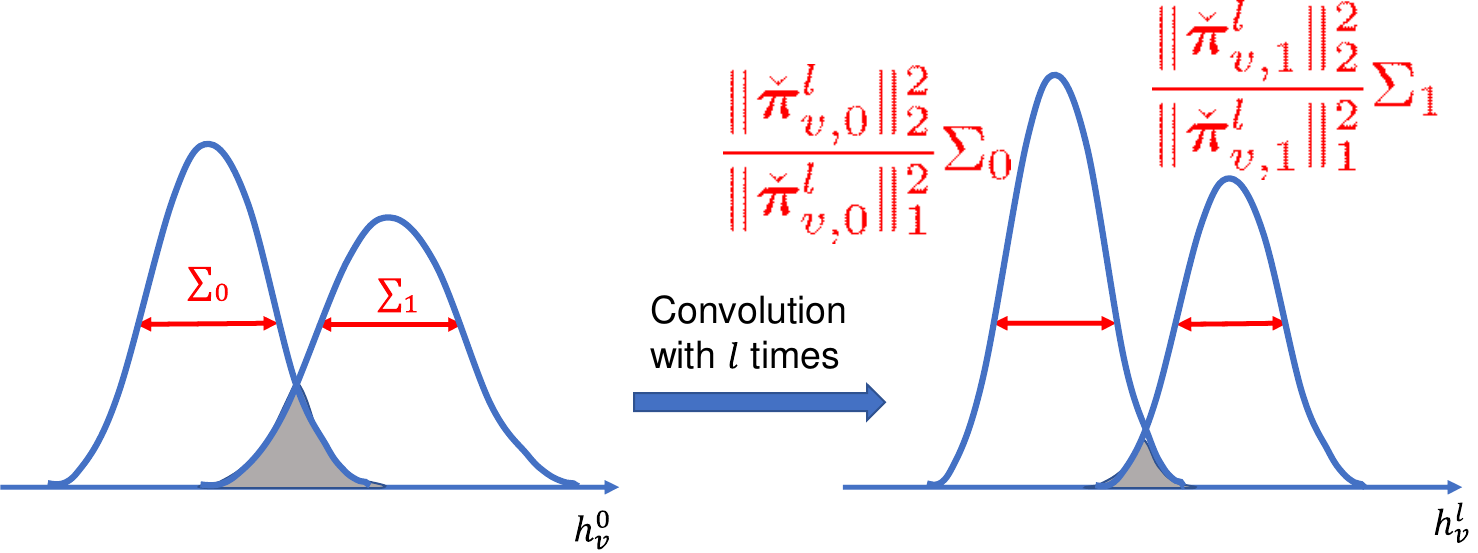}
 \vspace{-2mm}
	\caption{The overlaps of two distributions.}
        \label{variance0}
	\end{center}
\end{figure}

\subsection{Implementation of the Oracle in DeepGAT}

\begin{algorithm}[t]
\DontPrintSemicolon
\SetKwInOut{Input}{input}\SetKwInOut{Output}{output}
\Input{ $G=(V,E,X)$, $Y\in \mathbb{R}^{n \times |C|}$, and $L$ }
\Output{Predictions of class labels}
\nl $(\bm{h}_1^0,\bm{h}_2^0,\ldots,\bm{h}_n^0) \leftarrow X$ \;
\nl	\For{$l \in \{ 1,2,\ldots,L \}$}{
\nl     $Z\leftarrow  rm\_diag(\hat{A}^{l-1})Y$\;
\nl     $(\bm{z}_1,\bm{z}_2,\ldots,\bm{z}_n) \leftarrow Z^T$ \;
\nl		\For{ $v \in V$}{
\nl         $\bm{p}_v^{l-1} \leftarrow W^l_1 (\bm{h}_v^{l-1} || \bm{z}_v)$\; 
\nl			$\hat{\bm{y}}_v^{l-1} \leftarrow softmax(\bm{p}_v^{l-1})$ \;
		}
\nl		\For{$(v,u) \in E$}{
\nl			$\overline{\alpha}_{vu}^l \leftarrow att(\hat{\bm{y}}_v^{l-1}, \hat{\bm{y}}_u^{l-1}) $ \;
\nl			$\overline{\alpha}_{uv}^l \leftarrow \overline{\alpha}_{vu}^l$\;
		}
\nl		\For{ $v \in V$}{
\nl			\For{ $u \in N_v$}{
\nl				$\alpha_{vu}^l \leftarrow \frac{\exp(\overline{\alpha}_{vu}^l)}{\sum_{u\in N_v }\exp(\overline{\alpha}_{vu}^l)}$ 
			}
\nl				$\bm{a}_v^l \leftarrow \sum_{u \in N_v}\alpha_{vu} \bm{h}_u^{l-1}$ \;
\nl				$\bm{h}_v^l \leftarrow  \sigma(W_2^l\bm{a}_v^l)  $\;
		}
	}
\nl	\For{ $v \in V$}{
\nl		$\hat{\bm{y}}_v^{L} \leftarrow softmax(\bm{h}_v^{L})$
}		
\nl \Return $\{\hat{\bm{y}}_v^{l} \mid (v,l) \in V \times \{0,1,\ldots,L\}\}$
\caption{Feedforward procedures of DeepGAT}
\end{algorithm}

\begin{algorithm}[t]
\DontPrintSemicolon
\SetKwInOut{Input}{input}\SetKwInOut{Output}{output}
\Input{ $G=(V,E,X)$ and $L$ }
\Output{Predictions of class labels}
\nl $(\bm{h}_1^0,\bm{h}_2^0,\ldots,\bm{h}_n^0) \leftarrow X$ \;
\nl	\For{$l \in \{ 1,2,\ldots,L \}$}{
\nl		\For{$(v,u) \in E$}{
\nl			$\overline{\alpha}_{vu}^l \leftarrow att(W^l\bm{h}_v^{l-1}, W^l\bm{h}_u^{l-1}) $ \;
\nl			$\overline{\alpha}_{uv}^l \leftarrow \overline{\alpha}_{vu}^l$
		}
\nl		\For{ $v \in V$}{
\nl			\For{ $u \in N_v$}{
\nl				$\alpha_{vu}^l \leftarrow \frac{\exp(\overline{\alpha}_{vu}^l)}{\sum_{u\in N_v }\exp(\overline{\alpha}_{vu}^l)}$ 
			}
\nl				$\bm{a}_v^l \leftarrow \sum_{u \in N_v}\alpha_{vu} \bm{h}_u^{l-1}$ \;
\nl				$\bm{h}_v^l \leftarrow  \sigma(W^l\bm{a}_v^l)  $
	}
}
\nl	\For{ $v \in V$}{
\nl		 $\hat{\bm{y}}_v^L \leftarrow softmax(W\bm{h}_v^{L})$
	}
\nl	\Return $\{\hat{\bm{y}}_v^L \mid v \in V \}$
\caption{Feedforward procedures of GAT}
\end{algorithm}

This subsection discusses how to implement the oracle. Because a GAT with a few layers can predict class labels with reasonably high accuracy, the proposed DeepGAT method predicts the class labels of nodes in each layer. 
To achieve this, DeepGAT linearly transforms each  representation in each layer into a $|C|$-dimensional representation $\bm{p}_v^l$, 
and then computes $\hat{\bm{y}}_v^l=(\hat{y}_{v,0}^l,\hat{y}_{v,1}^l)^T \in \mathbb{R}^{|C|}=\mathbb{R}^2$ from the representations $\bm{p}_v^l$ in each layer, where $\hat{y}_{v,c}^l$ is the probability of $v$ belonging to class $c$ predicted in the $l$-th layer.
If the prediction $\hat{\bm{y}}_v^l$ is completely correct, then $\langle \hat{\bm{y}}_v^l, \hat{\bm{y}}_u^l \rangle $ coincides with $o(v,u)$. Therefore, let $\overline{\alpha}_{vu}^l$ be $\langle \hat{\bm{y}}_v^l, \hat{\bm{y}}_u^l \rangle $.

Algorithm~1 shows the pseudo code of the feedforward procedures of our DeepGAT.
For comparison, the pseudo-code for the conventional GAT is shown in Algorithm~2.
The correct convolution of DeepGAT depends on the accuracy of the prediction  $\hat{\bm{y}}_v^l$ at each layer.
Various previous researches have demonstrated that incorporating class label information  $Y\in \mathbb{R}^{n \times |C| }$ as additional input can improve GNN performance.
To facilitate the accuracy of this prediction,  we combine DeepGAT with this label propagation~\cite{SHeGNN, LabelAgg1, LabelAgg2,LabelAgg3}.
Here,  $Y$ is the raw label matrix whose rows corresponding to nodes $V'$ with graph-truth labels are one-hot vectors and the other rows for nodes  $V\setminus V'$ are filled with zeros.
$\hat{A}$ in Line~3 of Algorithm~1 is the row-normalized form of adjacency matrix of $G$, 
$\hat{A}^l$ is the $l$-th power of $\hat{A}$, and the function $rm\_diag$ sets the diagonal elements of the matrix to zero in order to avoid the label leakage~\cite{SHeGNN}.
In Line~6, the concatenation of the representation $\bm{h}_v^{l-1}$ and label aggregation $\bm{z}_v$ is linearly transformed with the learnable parameter $W_1^l$, and then 
the class labels for $v$ are predicted in Line~7. 
In Lines~9 and 10, attention coefficients are computed from the predicted class labels. 
The convolutions using the attention coefficients are repeated for $L$ layers, and class labels $\hat{\bm{y}}_v^L$ are predicted in the final layer.

DeepGAT learns $W_1^l$ and $W_2^l$ from graph-truth labels.
According to Subsection~\ref{subsec:4A}, if DeepGAT can correctly predict labels of nodes in each layer, the variance of representations in the final layer will be smaller.
Since prediction errors in the layers close to the input layer propagate toward the output layer, the correctness of predictions in the layers close to the output layer depends on the correctness of predictions in the layers close to the input layer. Therefore, in order to suppress prediction errors in the layers close to the input layer, we use the following loss function.
\begin{eqnarray}
&&-\sum_{l=1}^L\gamma(l)\underline{\sum_{v \in V'}[ y_v\log \hat{y}_{v,0}^l+(1-y_v)\log (1-\hat{y}_{v,1}) ] \label{newloss}} \\
&&\mbox{where } \gamma(l)=\frac{\delta}{l+\delta}+1 \nonumber 
\end{eqnarray}
The underlined part of Eq.~(\ref{newloss}) is the cross-entropy loss function for the $l$-th layer, and it is multiplied by $\gamma(l)$ which decreases monotonically with increasing the hyperparameter $\delta>0$ and converges to 1. By using the loss function~(\ref{newloss}), 
the proposed DeepGAT method aims to both
\begin{itemize}
\item minimize 
prediction error by minimizing Eq.~(\ref{newloss}) and 
\item reduce the variance of the distributions by Eq.~(\ref{eq:corollary1}).
\end{itemize}

\section{Experimental Evaluation}
\subsection{Experimental Setting}

\begin{table*}[t]
	\caption{Summary of benchmark datasets.}
	\begin{center}
		\begin{tabular}{|l|r|r|r|r|r|r|r|r|r|r|r|} \hline
			 Datasets  & \# of graphs & $|V|$   &  $|E|$    &  $d$  & $|C|$ & avg. & max.  & hub node  & diameter & density & avg. cluster \\
                &&    &    &   &  & degree &  degree & rate [\%] &  & $\frac{2|E|}{|V|(|V|-1)}$ & coefficient $cce$ \\ \hline
                Cora    &1& 2,708   & 10,556    & 1,433 & 7     &9.8    &338    &2.18\%	&19	&0.18\%	&24.07\%
 \\ \hline
                CiteSeer&1& 3,327	  & 9,104	  & 3,703 &	6     &7.5	&200	&1.23\%	&28	&0.11\%	&14.15\%
 \\ \hline
			CS      &1& 18,333  & 163,788   & 6,805 & 15    &19.9	&274	&18.04\%	&24	&0.05\%	&34.25\% \\ \hline
                Physics &1& 34,493  & 495,924   & 8,415 & 3     &	30.8	&766	&36.40\%	&17	&0.04\%	&37.76\% \\  \hline
			Flickr  &1& 89,250  & 899,756   & 500   & 7     &22.2	&10,852	&11.57\%	&8	&0.01\%	&3.30\% \\ \hline
			PPI     &24& 56,944  & 818,716   & 50    & 121  &28.3    &721    & 27.43\%&--&--& 27.06\% \\ \hline
 		\end{tabular}
	\end{center}
	\label{zikken1}
\end{table*}

\begin{figure}[b]
	\begin{center}
		\includegraphics[width=80mm]{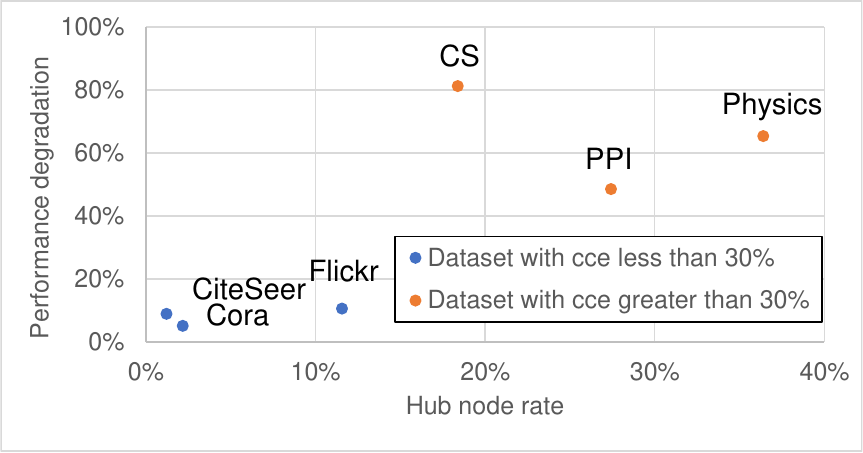}
	\end{center}
	\vspace{-5mm}
	\caption{Performance degradation of GAT.}
	\label{degradation}
\end{figure}

This section demonstrates that DeepGAT prevents over-smoothing by comparing DeepGAT with the original GAT. 
In the experiments in this paper, the proposed method is not compared with state-of-the-art GNNs, but only with GAT.
By comparing only with GAT, we verify whether predicting which class each node belongs to at each layer and convolving representations of nodes based on the prediction is effective. If we find that it is effective, we can use our ideas in various methods that use GAT.
Comparisons with the world's best GNNs and GNNs that have achieved depths of over 100 layers are included in future plans.
We implemented DeepGAT with Pytorch and Pytorch Geometric. 
See config.yaml at \url{https://github.com/JNKT215/DeepGAT_CANDAR} for the hyperparameter settings for each model. 

Table~\ref{zikken1} summarizes the benchmark datasets used in our experimental evaluation. Cora, CiteSeer, CS, and Physics are datasets for article citation networks, PPI is a network for protein-protein interactions, 
and Flickr is a network of images. 
We treated those networks as undirected graphs. PPI is a dataset for the multi-label node classification, and the other datasets are used for multi-class node classification. So, the label propagation in Lines~3 and 4 of Algorithm~1 was not used for PPI.
We treated nodes with degree greater than 30 as hub nodes, and the hub node rate was set to $\frac{|\{v\mid v\in V, |N_v|> 30\}|}{|V|}$.
Figure~\ref{degradation} shows the performance degradation of the GAT which is defined as the difference between the best micro-F1 score for various $L$ on each dataset and a micro-F1 score at the maximum $L$ that could be computed in an acceptable time.
The performance degradation of GAT with respect to CS and Physics is significant, while one with respect to Cora and CiteSeer is not significant. The former datasets have the hub node rates, and many of them also have higher cluster coefficients $cce$.
When there are more hub nodes in the graph, the average distance between any two nodes is smaller, so the number of layers for a feature at one node to propagate as a representation of another node is smaller. This makes over-smoothing more likely to occur.
In the next subsection, we report whether the proposed method can mitigate over-smoothing for CS, Physics, Flickr, and PPI whose hub node rates are more than 10\%.

\subsection{Experimental Results}
\subsubsection{Micro-F1 Scores with Various Numbers of Layers}
First, we confirmed that DeepGAT prevents over-smoothing when the number of layers is increased. 
Figure~\ref{acc_fig} shows the micro-F1 scores for DeepGAT and GAT with various numbers of layers and the CS, Physics, and Flickr datasets.
DP and SD indicate the dot-product attention and scaled dot-product attention, respectively.
To suppress prediction errors and reduce computation time for layers close to the input layer, label propagation in Lines 3 and 4 of Algorithm~1 was limited to the first through the third layers.

When $L$ is increased, the micro-F1 scores of GAT decrease significantly due to over-smoothing. 
In contrast, the DeepGAT with 15 layers for CS and Physics and one with 9 layers for Flickr maintain a relatively high micro-F1 score compared with the DeepGAT with 2 layers,
which indicates that DeepGAT has the ability to prevent over-smoothing.
The reason why the scores slightly decrease is that the same amount of data was used to train $W^1_1,W^2_1,\ldots,W^L_1,W^1_2,W^2_2,\ldots,W^L_2$ for DeepGATs with 2 and 15 layers.
Despite the fact that the difference between Algorithms~1 and 2 is slight, the difference in their performance is very large. 
Table~\ref{acc_table1} shows the best micro-F1 score (denoted by ``{\it best}'') for various $L$ on each dataset and a micro-F1 score (denoted by ``{\it max.}'') at the maximum $L$ that could be computed in an acceptable time.
The numbers in parentheses in Table~\ref{acc_table1} represent the numbers of layers $L$ at the time the scores were obtained, and the numbers in bold are the highest scores for each dataset.
Table~\ref{acc_table1} shows that DeepGAT is superior to GAT and the use of DeepGAT allows for the construction of deeper networks than GAT.

\begin{figure}
\centering
{\includegraphics[width=3.0in]{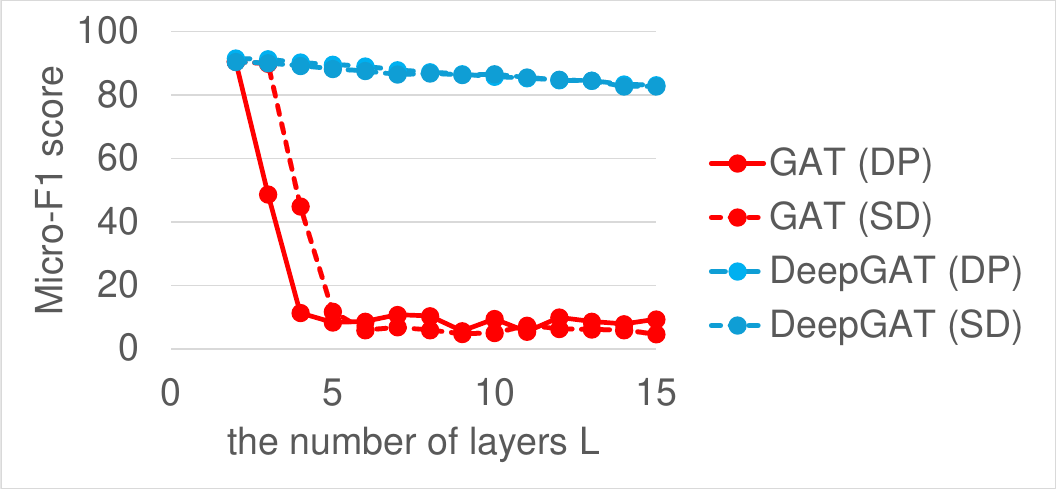}
\\(a) Micro-F1 scores for CS.
\vspace{2mm}
}
\\
{\includegraphics[width=3.0in]{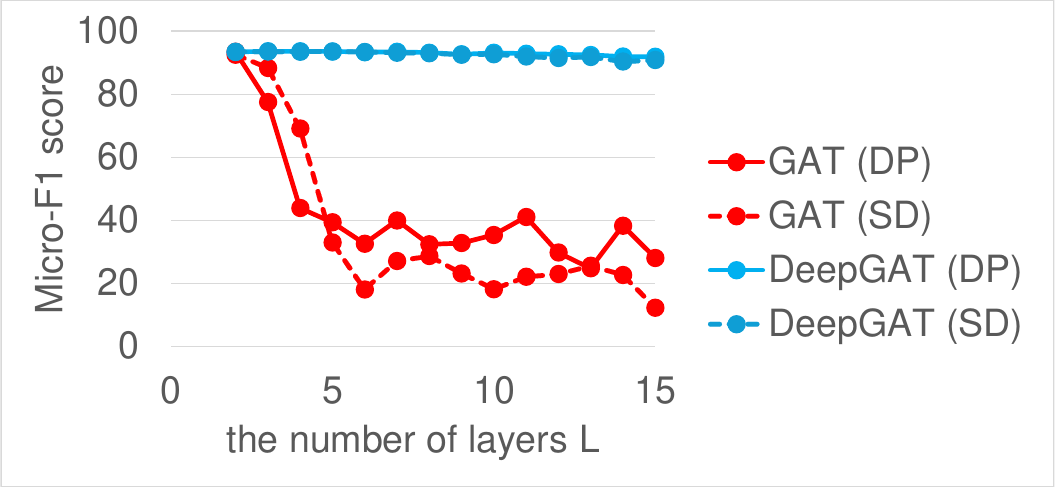}
\\(b) Micro-F1 scores for Physics.
\vspace{2mm}
}
\\
{\includegraphics[width=3.0in]{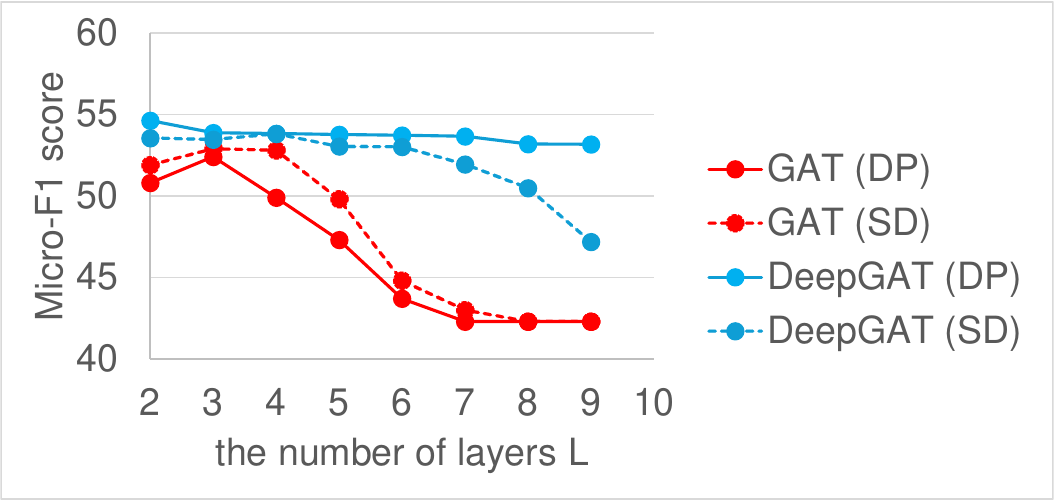}
\\(c) Micro-F1 scores for Flickr.
\vspace{2mm}
}
\\
{\includegraphics[width=3.0in]{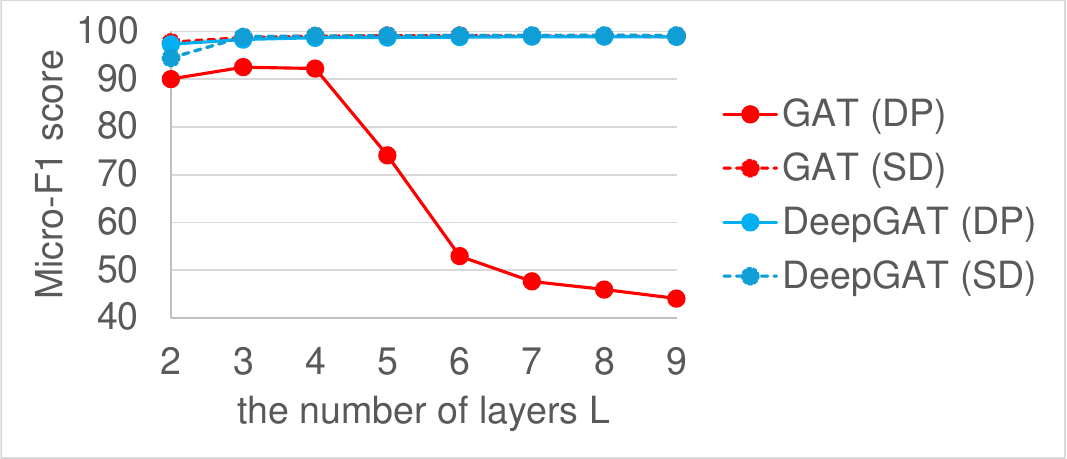}
\\(d) Micro-F1 scores for PPI.
}
\vspace{2mm}
\caption{Micro-F1 scores with various numbers of layers $L$ for the CS, Physics, Flickr, and PPI datasets.}
\label{acc_fig}
\end{figure}

\begin{table}[tb]
	\caption{Best Micro-F1 scores and Micro-F1 scores at the maximum $L$.}
	\vspace{-5mm}
	\begin{center}
		\begin{tabular}{|c|c|c|c|c|c|} \hline
			Model & &CS & Physics & Flickr & PPI  \\ \hline
			DeepGAT  & {\it best} &\textbf{91.5} (2) & \textbf{93.6} (4)  & \textbf{54.6} (2) & 98.9 (7) \\  \cline{2-6}
   
		    (DP)     & \cellcolor[gray]{0.8}{\it max.} &\cellcolor[gray]{0.8} 83.1 (15) & \cellcolor[gray]{0.8}91.8 (15)  & \cellcolor[gray]{0.8}53.2 (9) \cellcolor[gray]{0.8}&\cellcolor[gray]{0.8} 98.8 (9) \\ \hline
      
                DeepGAT  & {\it best} &90.4 (2) & \textbf{93.6} (5)  & 53.8 (4) & \textbf{99.2} (8,9) \\  \cline{2-6}
                
		        (SD)   &\cellcolor[gray]{0.8} {\it max.} &\cellcolor[gray]{0.8} 82.7 (15) & \cellcolor[gray]{0.8}90.8 (15)  &\cellcolor[gray]{0.8} 47.2 (9) & \cellcolor[gray]{0.8}\textbf{99.2} (9) \\ \hline
          
			GAT      & {\it best} &90.5 (2)  & 93.4 (2) & 52.4 (3) & 92.5 (3)  \\ \cline{2-6}
   
		      (DP)    & \cellcolor[gray]{0.8}{\it max.}   & \cellcolor[gray]{0.8}9.3 (15)  & \cellcolor[gray]{0.8}28.1 (15) &\cellcolor[gray]{0.8} 42.3 (9) &\cellcolor[gray]{0.8} 44.0 (9)  \\ \hline
        
                GAT     & {\it best} &90.6 (2)  & 92.5 (2) & 52.9 (3) & 99.1 (6)  \\ \cline{2-6}
                
		      (SD)    & \cellcolor[gray]{0.8}{\it max.} &\cellcolor[gray]{0.8} 4.6 (15)  &\cellcolor[gray]{0.8} 12.3 (15) &\cellcolor[gray]{0.8} 42.3 (9) & \cellcolor[gray]{0.8}98.9 (9)  \\ \hline
 		\end{tabular}
	\end{center}
	\label{acc_table1}
\end{table}

\begin{figure*}
    \begin{tabular}{ccc}
        \begin{minipage}{.25\textwidth}
            \centering
            \includegraphics[width=1\linewidth]{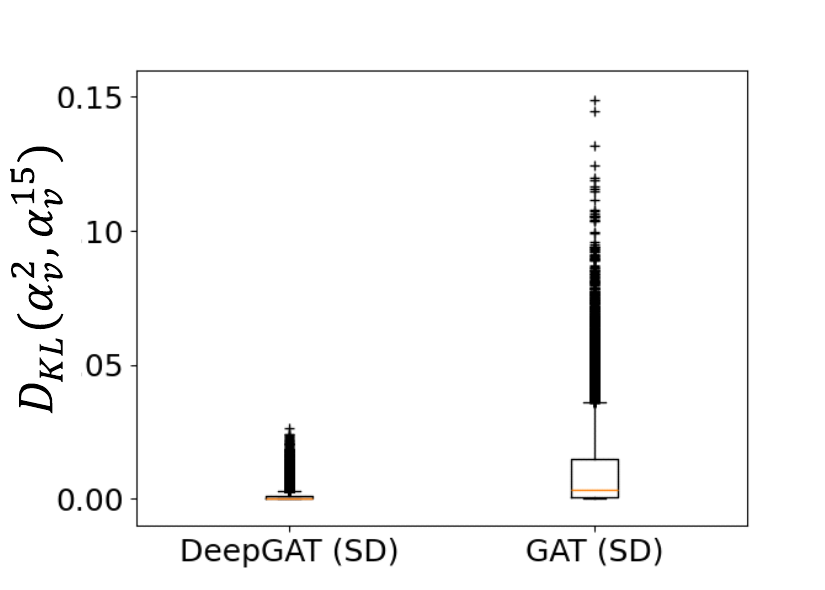}
            \\(a) CS
        \end{minipage}
        \begin{minipage}{.25\textwidth}
            \centering
            \includegraphics[width=1\linewidth]{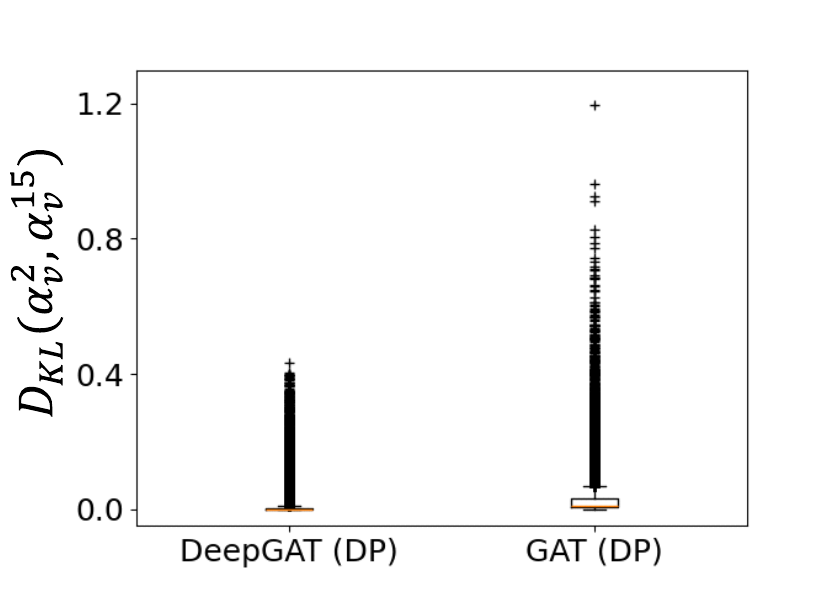}
            \\(b) Physics
        \end{minipage}
        \begin{minipage}{.25\textwidth}
            \centering
            \includegraphics[width=1\linewidth]{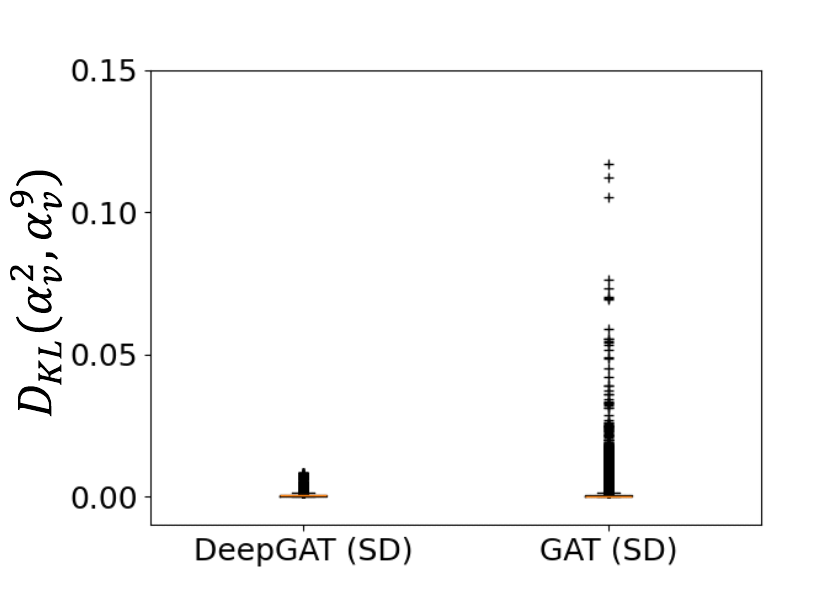}
            \\(c) Flickr
        \end{minipage}
        \begin{minipage}{.25\textwidth}
            \centering
            \includegraphics[width=1\linewidth]{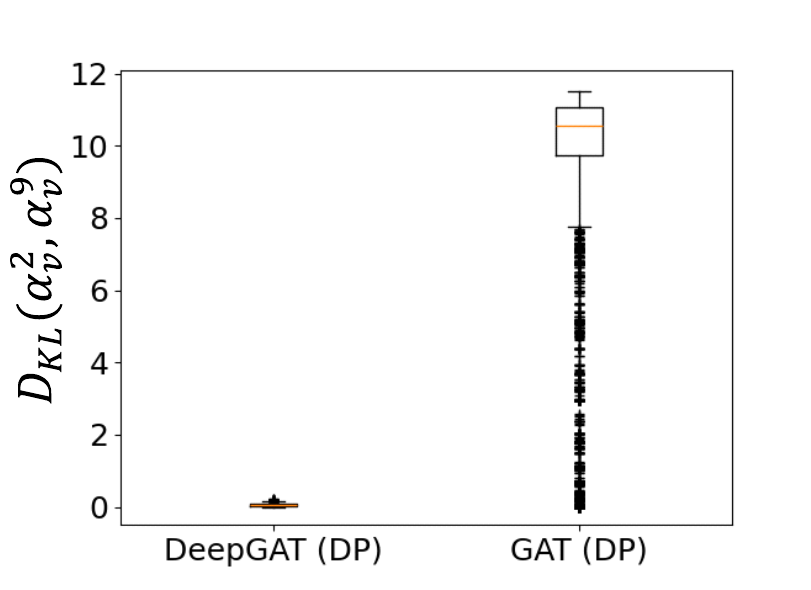}
            \\(d) PPI
        \end{minipage}
    \end{tabular}
\caption{Comparison between $D_{KL}(\bm{\alpha}_{v}^{2} || \bm{\alpha}_{v}^{L_{max}})$.}
	\label{attention_fig}
\end{figure*}

\subsubsection{Comparison between Attention Coefficients at 2 Layers}
Next, we confirmed that DeepGAT with a large number of layers can acquire appropriate attention coefficients as well as DeepGAT with a small number of layers.
In concrete terms, we compute the KL divergence between attention coefficients $\bm{\alpha}_v^2=(\alpha_{v1}^2,\alpha_{v2}^2,\ldots,\alpha_{v|V|}^2)^T$ for DeepGAT with 2 layers and $\bm{\alpha}_v^{L_{max}}$ for DeepGAT with $L_{max}$ layers.
If the impact of over-smoothing on DeepGAT is small, $\overline{\alpha}_{vu}^l$ should be almost the same as $o(v,u)$ for DeepGATs with both large and small numbers of layers.
Therefore, by measuring the difference between $\bm{\alpha}_v^2$ and $\bm{\alpha}_v^{L_{max}}$ by Eq.~(\ref{Divergence}), we can determine how accurately the features of the nodes adjacent to $v$ that belong to the same class as $v$ are convolved to the representation of $v$ in DeepGAT.
\begin{equation}
    D_{KL}(\bm{\alpha}_{v}^{L_1} || \bm{\alpha}_{v}^{L_2}) = \sum_{u\in N_v, \alpha_{vu}^{L_2} \ne 0}\alpha_{vu}^{L_1}\log\frac{\alpha_{vu}^{L_1}}{\alpha_{vu}^{L_2}} 
    \label{Divergence}
\end{equation}
Figure~\ref{attention_fig} shows a box-and-whisker diagram of $D_{KL}(\bm{\alpha}_{v}^{2} || \bm{\alpha}_{v}^{L_{max}})$ computed for various nodes $v\in V$.
The figure shows that the variance of $D_{KL}(\bm{\alpha}_{v}^{2} || \bm{\alpha}_{v}^{L_{max}})$ for DeepGAT is smaller than that for GAT.
Thus, DeepGAT with a large number of layers is trained to acquire similar attention coefficients to DeepGAT with a small number of layers.
These results show that for each node $v$, DeepGAT with larger layers convolves only a small number of features of $v$'s neighborhood belonging to a different class than $v$ into the representation of $v$ compared with the original~GAT.

\subsubsection{Error Rate $e_N$ using the Nearest Neighbor Method}
\begin{figure}
\centering
{\includegraphics[width=3.0in]{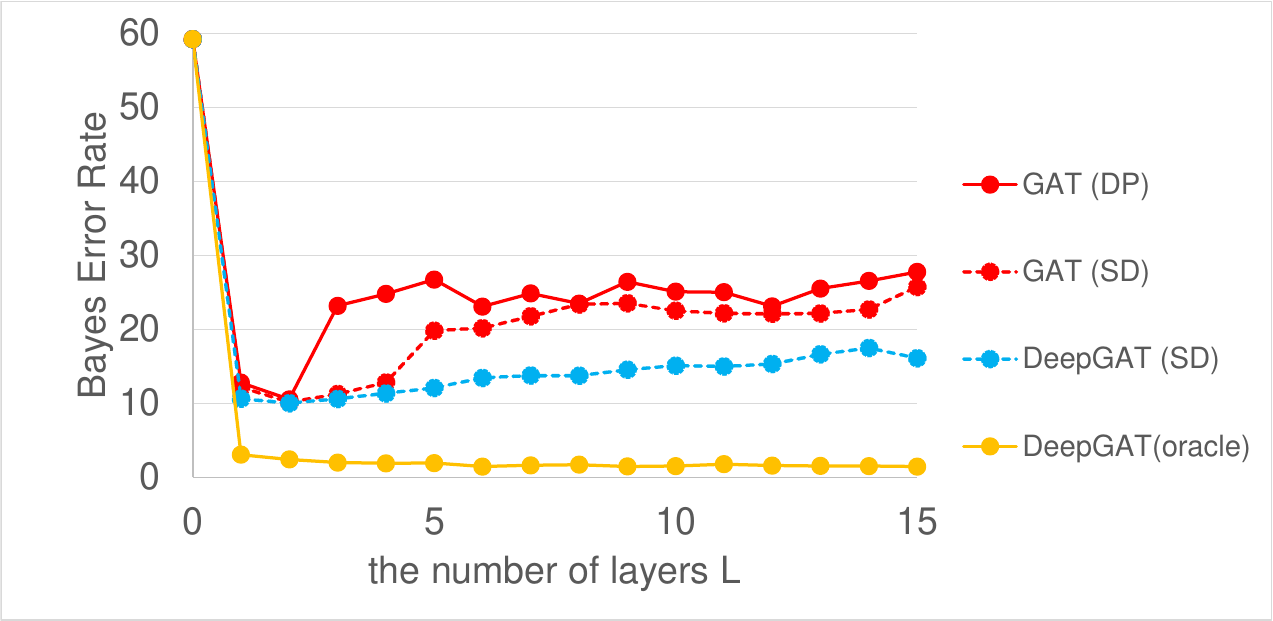}
\\(a) CS
\vspace{2mm}
}
\\
{\includegraphics[width=3.0in]{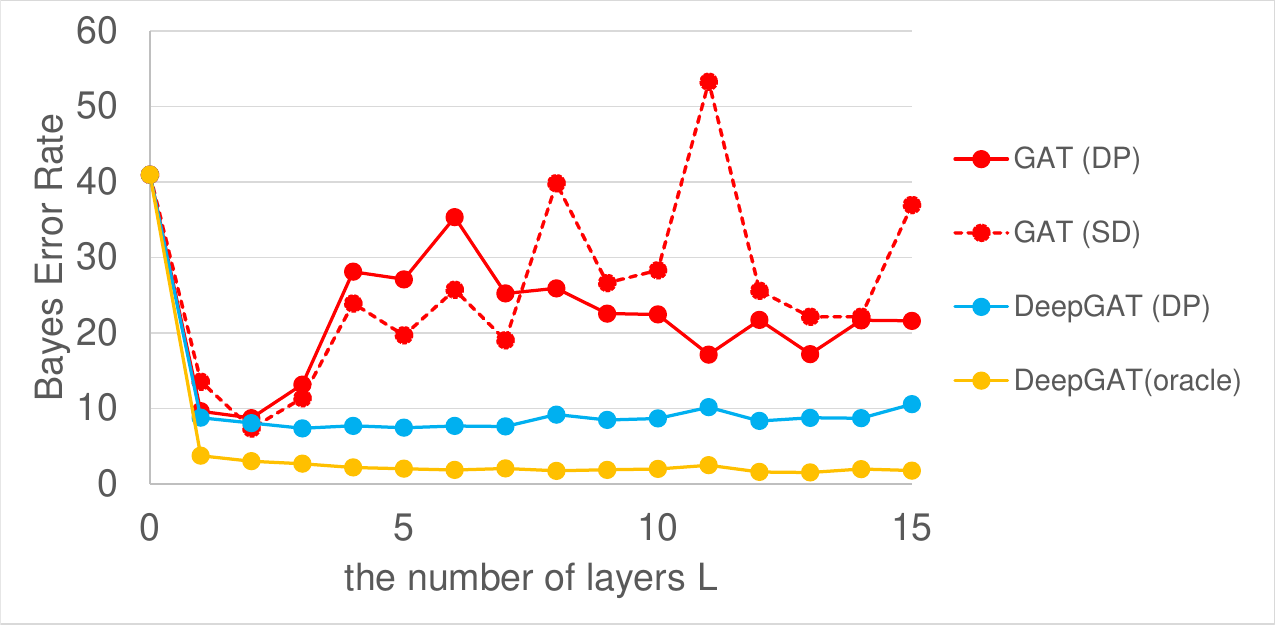}
\\(b) Physics
\vspace{2mm}
}
\\
{\includegraphics[width=3.0in]{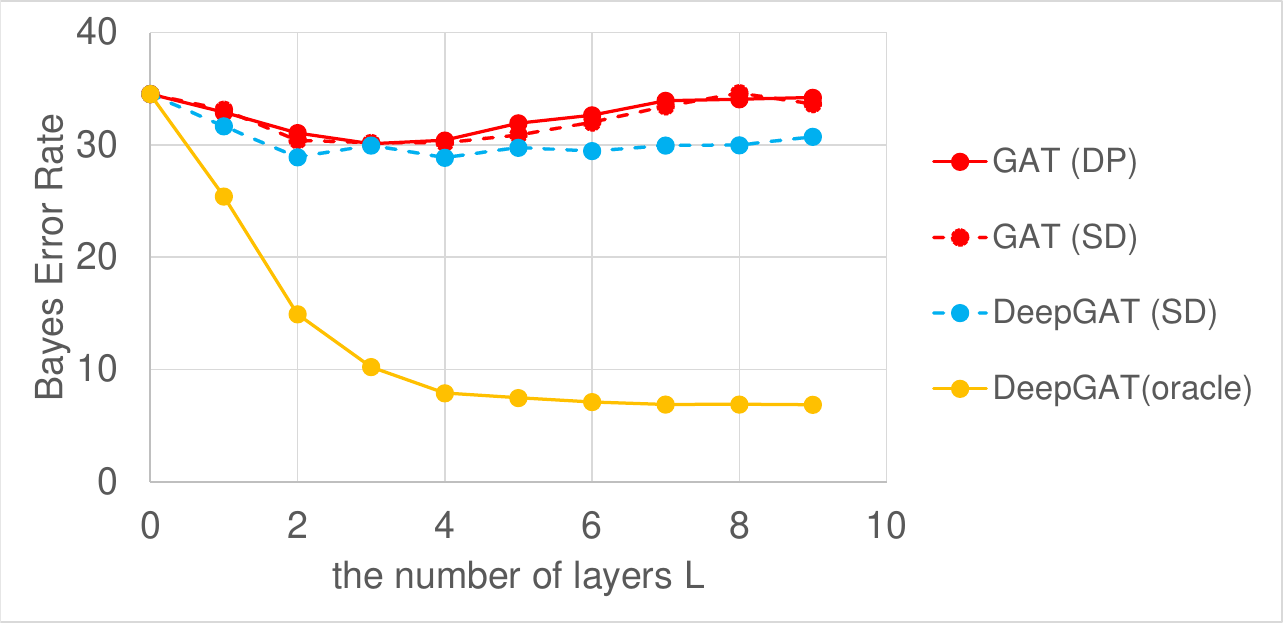}
\\(c) Flickr
\vspace{2mm}
}
\caption{Error rate $e_N$ using the nearest neighbor method. }
	\label{variance_fig}
\end{figure}

In the discussion immediately following the proof of Lemma~\ref{corollary1}, we noted that in the presence of an oracle, the variance of the distribution that the representation $\bm{h}_v^l$ follows decreases as $l$ increases.
The smaller the overlap between the distribution $\mathcal{D}_0$ followed by the representation $\bm{h}_v^L$ of node $v$ belonging to class 0 and the distribution $\mathcal{D}_1$ followed by the representation $\bm{h}_u^L$ of node $u$ belonging to class 1, the more accurate we can expect the classification to be.
The overlap is called the Bayes error rate and it is the lowest possible error rate for any classifiers for the given dataset.
A low rate suggests that a useful representation is obtained for classification by learning the representation.
In practice, however, it is difficult to know the exact probability distribution $\mathcal{D}_c$ from $\bm{h}_v^L$.
Thus, we estimate the upper and lower bounds of the Bayes error rate $e_B$ with an error rate $e_N$ by the nearest neighbor method with a large number of prototypes in the low-dimensional space. 
According to previous studies \cite{Thomas1967,Fukunaga1987},
\[e_B \le e_N \le e_B(2-\pi e_B) \le 2e_B, \]
where $\pi=\frac{|C|}{|C|-1}$. Therefore, the upper and lower bounds of the Bayes error rate $e_B$ are given by
\[\frac{1}{2}e_N \le \frac{1-\sqrt{1-\pi e_N}}{\pi} \le e_B \le e_N\]
for $0 \le e_N \le 0.5$.
Figure~\ref{variance_fig} shows $e_N$ for the CS, Physics and Flickr datasets when the number of layers $L$ was increased. Figure~\ref{variance_fig} does not include $e_N$ for PPI, because PPI is the dataset for the multi-label node classification. 
In Fig.~\ref{variance_fig}, the yellow broken line is obtained by replacing $att(\hat{\bm{y}}_v^{l-1}, \hat{\bm{y}}_u^{l-1})$ in Line~5 of Algorithm~1 into the oracle $o(v,u)$.
The error rate represented by the line gradually decreases as $L$ increases, which demonstrates that the variance of the distribution $\mathcal{D}_c$ for each class $c$ decreases as discussed after the proof of Lemma~\ref{corollary1}.
In contrast, when $\overline{\alpha}_{vu}$ is computed from representations but not the oracle, $e_N$ increases slightly with an increase of $L$ from 1. 
This is because $\overline{\alpha}_{vu}$ computed from the learned representation does not exactly match $o(v,u)$. 
However, $e_N$ for DeepGAT is always smaller than for GAT, which indicates that the overlap between the distributions $\mathcal{D}_0$ and $\mathcal{D}_1$ followed by representations $\bm{h}_v^L$ learned from DeepGAT is small.

\section{Conclusion}
GNNs tend to suffer from over-smoothing leading to performance degradation. 
Therefore, this paper introduced DeepGAT for predicting the class to which network nodes belong in a deep GAT.
It avoids over-smoothing in a GAT by ensuring that nodes in different classes are not similar at each layer. 
Using DeepGAT to predict class labels, a 15-layer network could be constructed without the need to tune the number of layers. 
DeepGAT prevented over-smoothing and achieved a 15-layer GAT with similar performance to a 2-layer GAT, as indicated by the similar attention coefficients.
DeepGAT enables the training of a large network to acquire similar attention coefficients to a network with few layers. It avoids the over-smoothing problem and obviates the need to tune the number of layers, thus saving time and enhancing GNN performance.


\end{document}